\documentclass[reqno,12pt]{amsart}

\usepackage{enumerate}
\usepackage[margin=1in]{geometry}
\usepackage{enumitem}
\usepackage{ifpdf}
\usepackage{amsmath}
\usepackage{amsfonts}
\usepackage{amssymb}
\usepackage{amsaddr}
\usepackage{amsthm}
\usepackage{amsrefs}
\usepackage{mathrsfs}
\usepackage{cite}
\usepackage[hyperfootnotes=false]{hyperref}
\usepackage[dotinlabels]{titletoc}
\usepackage{nicefrac}
\usepackage{float}
\usepackage[multiple]{footmisc}
\usepackage{graphicx}
\usepackage{dcolumn}
\usepackage{tabu}

\newtheorem{thm}{Theorem}[section]

\theoremstyle{definition}
\newtheorem{defn}[thm]{Definition}

\theoremstyle{remark}

\theoremstyle{note}

\newcolumntype{A}{D{.}{.}{2.3}}

\makeatletter

\newcommand{\Rmnum}[1]{\expandafter\@slowromancap\romannumeral #1@}
\makeatother
\DeclareMathOperator*{\argmax}{arg \, max}

\author{David A. Meyer and Asif Shakeel}
\address{Department of Mathematics, University of California/San Diego, La Jolla, CA 92093-0112, USA}
\email{dmeyer@math.ucsd.edu, ashakeel@ucsd.edu}
\date{\today}

\ifpdf
\fi

\title[Estimating an ADHMM] {Estimating an Activity Driven Hidden Markov Model}

\begin{document}

\begin{abstract}
We define a Hidden Markov Model (HMM) in which each hidden state has 
time-dependent {\sl activity levels\/} that drive transitions and 
emissions, and show how to estimate its parameters.  Our construction
is motivated by the problem of inferring human mobility on sub-daily 
time scales from, for example, mobile phone records.
\end{abstract}

\maketitle

\section{Introduction} \label{section:intro}
  
Hidden Markov models (HMMs) are stochastic models for systems with a 
set of unobserved states between which the system hops stochastically, 
sometimes emitting a signal from some alphabet, with probabilities 
that depend upon the current state.  The situation in which we are
specifically interested is human mobility, partially observed, 
{\it i.e.}, occasional signals about a person's location.  For 
example, consider the cells of a mobile phone network, from which a 
user can make calls.  In this case the states of a HMM are the cells, 
and the emitted signals are the cell itself, if a call is made by a 
particular user during each of a sequence of time intervals, or 
nothing (0), if that user does not make a call.  In the latter case, 
the state (location) of the user is `hidden', and must be inferred, 
while in the former case, assuming no errors in the data, the `hidden' 
state is revealed by the call record.\footnote{Load balancing, in 
which calls may be routed through cell towers that are not the 
closest, makes this not strictly true.  The general model we consider 
here allows for this possibility.}  Since these are data from a 
{\sl mobile\/} phone network, a user can move from cell to cell.

Although many analyses of human mobility have estimated no more than 
rather crude statistics like the radius of gyration, the fraction of 
time spent at each location, or the entropy of the timeseries of 
locations~\cite{GHB,SQBB,Csajietal,Lenormandetal}, others have used 
HMMs to describe partially observed human mobility and have estimated 
their parameters~\cite{FLM,MRM,Perkinsetal}.  With short time steps, 
however, a standard HMM (with time-independent parameters) is not a
plausible model, since human mobility behavior changes according to, 
for example, the time of day~\cite{KCPN,SQBB,Csajietal,CAAMG}.  We 
would like to create, therefore, a HMM with time-{\sl dependent\/} 
parameters.  Of course, allowing, for example, arbitrary 
transition/emission probabilities at each time step, would lead to an 
extremely underdetermined model.  Rather, we need a model with only a 
few additional parameters to capture the time-dependence of human 
mobility.  Since the total numbers of trips~\cite{KCPN,SQBB,CAAMG} 
and mobile phone calls~\cite{Csajietal,DMRRS} vary with time of day 
and day of week, we develop a time-dependent HMM in which the 
non-trivial transition and emission probabilities are proportional to
{\sl activity levels}, {\it i.e.}, to some given functions modeling 
how active humans are at different times and places.

Since the transition and emission probabilities in our HMM are not
constant in time, it is a {\sl non-stationary\/} HMM.  Many 
generalizations of HMMs have been considered previously, of course, as
more faithful models of various real systems.  Some of these are 
non-stationary:  Deng, for example, considers a class of models in 
which the emission probabilities are somewhat non-Markovian, depending
on a number of previous emissions, and also have polynomial-in-time
trend components which are to be estimated~\cite{Deng}.  
{\sl Duration\/} HMMs (DHMMs), first suggested by 
Ferguson~\cite{Ferguson}, allow a sort of non-stationarity in the 
state transition process by including a randomly chosen duration each 
time the state changes, {\it i.e.}, a number of time steps without a 
transition away from that state.  This kind of model has been 
generalized to make the transition probabilities functions of the 
number of steps the system has been in the current 
state~\cite{SinKim}.  In a different direction, since one can think of
transitions between the hidden states with different emission 
probability distributions as a kind of non-stationarity, {\sl triplet
Markov chains\/} (TMCs) include an auxiliary set of underlying states, 
each of which corresponds to a different stationary regime for a 
HMM~\cite{LanchantinPieczynski}.  Our approach is different than that
of DHMMs and TMCs in that the time dependence of the transition and
emission probabilities is not intrinsic and random, but rather 
exogenous and deterministic.  Furthermore, unlike Deng's 
models~\cite{Deng}, we take the ``trend'' part of the time dependence 
to be given, not an additional (set of) parameter(s) to be estimated.

\smallskip
Formally, our model consists of $N$ possible hidden states, and we 
denote by $(X_t)\in [N]^T$ the time series of $T$ hidden states (for 
any $n\in\mathbb{N}$, $[n]=\{1,\ldots,n\}$).  State transitions happen 
according to a sequence of matrices giving the conditional 
probabilities of transitions,
\begin{equation*} 
 A(t) = \bigl(a_{ij}(t)\bigr),{\rm\ where\ } 
 a_{ij}(t) = \Pr(X_{t+1} = i \mid X_t = j).
\end{equation*}
At each state we observe an emission that takes a value from the set 
$[M]\cup\{0\}$, where $0$ denotes ``absence of an emission''.  Let 
$y=(y_t)\in ([M]\cup\{0\})^T$ be the series of observed emissions.  
The probability of emission $s$ at time $t$, from state $j$, is
\begin{equation*} 
 b_{s j}(t) = \text{Pr}(Y_t=s\mid X_t =j).
\end{equation*} 

We define time varying {\sl activity levels\/} for state $j$ by a 
pair of functions with non-negative values, $(f_j,g_j):[T]\to [0,1]^2$, 
the {\sl activity functions}, which modulate transitions and emissions 
from the state, respectively.  Given a state $j$, the transition 
probabilities from that state are functions of {\sl transition 
parameters\/} 
$(\tau_{ij}\ge0)$, $i\neq j\in [N]$, and the activity level:
\begin{equation} \label{aij}
 a_{ij}(t) =  
 \left\{ 
  \begin{array} {ll}
   f_j(t) \tau_{ij}                             &\text{if } i\neq j; \\
   1- f_j(t) \sum_{i\in [N], i\neq j} \tau_{ij} &\text{if } i=j,
  \end{array}
 \right.
\end{equation}   
subject to the constraints that for all $j\in [N]$ and all $t\in[T]$,
\begin{equation} \label{maxconstrtaut}
 f_{j}(t)\sum_{i\in [N], i\neq j}\tau_{ij} \leq 1.                                                               
\end{equation}
In practice we may have {\it a priori\/} knowledge that some 
transitions do not occur, so that $(a_{ij})$ (and $(\tau_{ij})$) have
some entries set to $0$.  For example, with 10 minute time intervals 
and states representing cells in a mobile phone network, transitions 
between sufficiently distant cells are precluded.

Similarly, we assign emission parameters 
$(\epsilon_{s j}\ge0)$, $s\in [M]$, to each state $j$.  The emission 
probabilities are
\begin{equation} \label{bsj}
 b_{s j}(t) = 
 \left\{
  \begin{array} {ll}
   g_j(t)\epsilon_{s j}                  &\text{if } s\neq0; \\
   1-g_j(t)\sum_{s\in [M]}\epsilon_{s j} &\text{if } s=0,
  \end{array}
 \right.
\end{equation}
subject to the constraints that for all $j\in [N]$ and all $t\in[T]$,
\begin{equation} \label{maxconstrepst}
 g_j(t)\sum_{s\in [M]}\epsilon_{s j}\leq1.                                                              
\end{equation}

Denote the initial distribution over states by $\pi_j = \Pr(X_1 = j)$, 
subject to the constraints $\pi_{j} \geq 0$ and
\begin{equation} \label{maxconstrpi}
 \sum_{j\in [N]} \pi_j = 1.
\end{equation} 

Were this a typical hidden Markov model, we could estimate its 
parameters using the Baum-Welch algorithm~\cite{BaumEagon,Welch}.  
Since it is not, we develop a novel Expectation Maximization (EM) 
algorithm~\cite{DLR} to estimate the parameters 
$\Theta = \big((\pi_j),(\tau_{ij}),(\epsilon_{s j})\big)$, given $y$, 
$f_j(t)$ and $g_{j}(t)$, as follows.

\section{Expectation Maximization} \label{sec:em}

The expectation maximization algorithm maximizes, at each iterative 
step, the (expected) log-likelihood function described below.  Let 
$\mathcal{X}$ be the set of all possible time series of states and let 
${\hat\Theta}^k$ be the estimate of $\Theta$ at the $k$-th iteration 
of the algorithm. 
\begin{equation} \label{thetatuple}
 {\hat\Theta}^k 
  = \big((\hat\pi^k_j),(\hat\tau^k_{ij}),(\hat\epsilon^k_{sj})\big).
\end{equation}
The algorithm begins by initializing the parameter estimates in the 
first ($k=1$) iteration.  Then the $k+1^{\rm st}$ iteration consists 
of two steps:
\begin{enumerate}[label=(\roman{*})]
\item \label{itr1} Compute the {\sl expectation value\/} of the 
log-likelihood, using the current ($k^{\rm th}$) estimate for the 
parameters:
\begin{equation} \label{likerat}
 \mathcal{L}(\Theta,{\hat\Theta}^k) 
  = \sum_{x\in\mathcal{X}}\log[\Pr(x,y;\Theta)] 
                          \Pr(x\mid y;{\hat\Theta}^k),
\end{equation}
where $\Pr(\cdot;\Theta)$ means $\Pr(\cdot)$ in a probability 
distribution parametrized by $\Theta$.
\item \label{itr2} Find the parameters that {\sl maximize\/} the
expected log-likelihood: 
\begin{equation*} 
 {\hat\Theta}^{k+1}
  = \underset{\Theta}{\argmax} \:\mathcal{L}(\Theta,{\hat\Theta}^k),
\end{equation*}
subject to constraints in the inequalities~\eqref{maxconstrtaut},
\eqref{maxconstrepst} and~\eqref{maxconstrpi}.
\end{enumerate}
As in the regular Baum-Welch algorithm, we express our computations in 
terms of certain conditional probabilities based on the parameters 
estimated at the $k^{\rm th}$ iteration, 
\begin{equation} \label{gammaxi}
\begin{aligned} 
  \gamma^k_j(t) &= \Pr(X_t = j\mid y;{\hat\Theta}^k),               \\
 \xi^k_{i j}(t) &= \Pr(X_t = j, X_{t+1} = i\mid y;{\hat\Theta}^k).  \\            
\end{aligned}
\end{equation}
Some reindexing of eq.~\eqref{likerat} yields the following expression 
for $\mathcal{L}(\Theta,{\hat\Theta}^k)$ in terms of these 
probabilities:
\begin{equation} \label{Lexp}
 \mathcal{L}(\Theta,{\hat \Theta}^k) 
  = \sum_{j\in[N]}\log(\pi_j)\gamma^k_j(1) 
    +\sum_{i,j\in[N]}\sum_{t=1}^{T-1}
     \log\bigl(a_{ij}(t)\bigr)\xi^k_{i j}(t) 
    +\sum_{j\in[N]}\sum_{t=1}^{T}
     \log\bigl(b_{y_t j}(t)\bigr)\gamma^k_j(t).
\end{equation} 
We iterate steps~\ref{itr1} and~\ref{itr2}, for which we compute 
$\gamma^k_j(t)$ and $\xi^k_{i j}(t)$ in eqs.~\eqref{gammaxi}, and 
$\mathcal{L}(\Theta,{\hat\Theta}^k)$ in eq.~\eqref{Lexp} above.  We 
continue until some standard of convergence is achieved.  We then 
output the final ${\hat\Theta}^k$ as our estimate of $\Theta$. 

\begin{thm} \label{mainthm}
There is a constrained expectation maximization algorithm giving a 
sequence of estimates $\hat\Theta^k$ that converges to a critical 
point of the likelihood function, which is the maximum likelihood 
estimate $\hat\Theta$ for the observed sequence $y$ when the initial 
guess is sufficiently close.  Further, to achieve  a precision 
$\epsilon$ in the estimates, the time complexity of the algorithm is  
$O\bigl((N^2 + M)T\log( T/\epsilon )\bigr)$.  In particular, 
for a fixed precision $\epsilon$, the time complexity is 
$O\bigl((N^2 + M)T\log T\bigr)$.
\end{thm}
\begin{proof}
Suppose we have the estimates  $\hat\Theta^{k}$ defined in eq.~\eqref{thetatuple} 
from step $k$ of the algorithm.   We proceed to compute 
$\gamma^{k+1}_j(t)$ and $\xi^{k+1}_{ij}(t)$ in eqs.~\eqref{gammaxi} to 
begin the next iteration.  Just as in the regular Baum-Welch 
algorithm, we apply dynamic programming.  Denote the $k^{\rm th}$ 
estimate of the transition matrix by 
$\hat A^k(t) = \bigl(\hat a^k_{ij}(t)\bigr)$, where 
\begin{equation} \label{hataij}
 \hat a^k_{ij}(t)
 =  
 \left\{\begin{array} {ll}
        f_j(t) \hat\tau^k_{ij}                   &\text{if } i\neq j; \\
        1 -  f_j(t)\sum_{i\neq j}\hat\tau^k_{ij} &\text{if } i=j.
        \end{array}
 \right .
\end{equation}   
Similarly,
\begin{equation} \label{hatbsj}
 \hat b^k_{sj}(t) 
 =  
 \left\{\begin{array} {ll}
        g_j(t) \hat\epsilon^k_{sj}                    &\text{if } s\neq0; \\
        1 - g_j(t)\sum_{s\in [M]} \hat\epsilon^k_{sj} &\text{if } s=0.
        \end{array}
 \right .
\end{equation}  
It is  convenient to define $\hat B^k(t)$ to be the diagonal matrix with
$jj^{\rm th}$ entry $\hat b^k_{y_t j}(t)$.
Now compute two sequences of (co)vectors, $\alpha^k(t)\in\mathbb{R}^N$ 
and $\beta^k(t)\in(\mathbb{R}^N)^{\dagger}$, recursively, as follows:
\renewcommand{\arraystretch}{1.5}
\begin{align*}
 \alpha^k(1) 
 &=  
 \hat B^k(1)\hat\pi^k;\\
 \alpha^k(t) 
 &=
 \hat B^k(t)\hat A^k(t-1)\alpha^k(t-1);\\
 \beta^k(T) 
 &=
 {\bf 1}^\mathsf{T} \hat B^k(T); \\
 \beta^k(t) 
 &=
 \beta^k(t+1)\hat A^k(t)\hat B^k(t).
\end{align*}
Then\footnote{With a slight  notational  abuse, $\hat B^k(t)^{-1}$ denotes the diagonal matrix whose
$jj^{\rm th}$ entry is $1/\hat b^k_{y_t j}(t)$ if $\hat b^k_{y_t j}(t) \neq 0$ and $1$ otherwise.}
\begin{align*}
 \gamma^{k+1}_j(t) 
 &=
 {\beta^k_j(t)\hat b^k_{y_t j}(t)^{-1}\alpha^k_j(t)    
 \over
 \beta^k(t)\hat B^k(t)^{-1}\alpha^k(t)
};                                                                 \\
\xi^{k+1}_{i j}(t)
 &=
{\beta^k_i(t+1)\hat a^k_{i j}(t)\alpha^k_j(t)
 \over
 \beta^k(t+1)\hat A^k(t)\alpha^k(t)
}.                                                                 \\
\end{align*}

\smallskip
The estimates for the initial probabilities $\pi_j$ are the same as in 
the normal Baum-Welch algorithm, as is clear from the expression for 
$\mathcal{L}(\Theta,{\hat\Theta}^k)$ in eq.~\eqref{Lexp}.  Thus 
\begin{equation*}
 \hat \pi^{k+1}_j= \gamma^{k+1}_j(1).
\end{equation*} 
 
Now notice that all the constraints in~\eqref{maxconstrtaut} 
and~\eqref{maxconstrepst}  necessary to define $\mathcal{L}(\Theta,{\hat\Theta}^k)$ are implied by the strongest constraints:
for all $j\in [N]$,
\begin{equation} \label{maxconstrtau}
 f^*_{j}\sum_{r\neq j}\tau_{rj} \leq 1,                                                               
\end{equation}
where $f^*_{j} = \max_{t\in [T-1]} \: f_{j}(t)$, and
\begin{equation} \label{maxconstreps}
 g^*_{j}   \sum_{s\in [M]} \epsilon_{s j} \leq 1,                                                              
\end{equation}
where  $g^*_{j} = \max_{t\in [T]} \: g_{j}(t)$.

Consider the computation of $\hat\tau^{k+1}_{ij}$, for 
$i\neq j\in [N]$.  It should lie in the domain 
$F_j\subset\mathbb{R}^N$ defined by  
constraints~\eqref{maxconstrtau} and the non-negativity of the 
parameters.  Since these constraints are independent for different 
$j$s, we can consider each $j$ separately, and find the optimal 
parameters $\tau_{ij}$ by computing the critical points of 
$\mathcal{L}(\Theta,{\hat\Theta}^k)$ relative to $(\tau_{ij})\in F_j$.
Using eqs.~\eqref{Lexp} and~\eqref{aij}, 
\begin{equation} \label{partialtau}
 \frac{\partial\mathcal{L}(\Theta,{\hat\Theta}^k)}
      {\partial \tau_{ij}} 
 =  
 \frac{1}{\tau_{ij}}
 \sum_{t=1}^{T-1}\xi^k_{ij}(t)
 - \sum_{t=1}^{T-1}\frac{f_j(t)\xi^k_{jj}(t)}
                        {1-f_j(t)\sum_{r\neq j}\tau_{rj}} 
\end{equation}
If the left sum in eq.~\eqref{partialtau}, 
$\sum_{t=1}^{T-1}\xi^k_{ij}(t) = 0$, the derivative is nonpositive, 
so $\mathcal{L}(\Theta,{\hat\Theta}^k)$ is weakly decreasing and 
$\tau_{ij} = 0$ gives its largest value.  If the right sum in 
eq.~\eqref{partialtau}, 
$\sum_{t=1}^{T-1}f_j(t)\xi^k_{jj}(t)/
                 (1-f_j(t)\sum_{r\neq j}\tau_{rj}) = 0$, the 
derivative is nonnegative, so $\mathcal{L}(\Theta,{\hat\Theta}^k)$ is
weakly increasing and takes its maximum value when $\tau_{ij}$ is as
large as possible, {\it i.e.}, when it saturates 
constraint~\eqref{maxconstrtau}.  We will show how to handle this
situation after discussing the generic case which we do next. 

Assuming then that neither sum in eq.~\eqref{partialtau} is $0$, to 
find the stationary points of $\mathcal{L}(\Theta,{\hat\Theta}^k)$ we 
set eq.~\eqref{partialtau} to $0$ and solve for $\tau_{ij}$.  
Specifically, $\tau_{ij}$ must satisfy
\begin{equation} \label{sumeq}
 \frac{1}{\tau_{ij}}\sum_{t=1}^{T-1}\xi^k_{i j}(t) 
 = 
 \sum_{t=1}^{T-1}\frac{f_j(t)\xi^k_{jj}(t)}
                      {1-f_j(t)\sum_{r\neq j}\tau_{rj}}.
\end{equation}
We note that if $f_j(t)\equiv1$, which makes the transition 
probabilities, $a_{ij}$, time independent, then the solution 
to eq.~\eqref{sumeq} is the familiar Baum-Welch solution: 
$\hat a^k_{ij} = \hat \tau^k_{ij}
= \sum_{t=1}^{T-1}\xi^k_{ij}(t)/
  \sum_{t=1}^{T-1}\gamma_j(t)$ for all $i,j\in[N]$.  For non-constant 
activity functions, however, the solution is more complicated.

Since the right side of eq.~\eqref{sumeq} is manifestly independent of 
$i$, the left side must be, too.  Let
\begin{equation} \label{tautilde}
 \tau_{j} = \frac{\tau_{ij}}{\sum_{t=1}^{T-1}\xi^k_{ij}(t)} 
          = \frac{\tau_{ij}}{\Xi^k_{ij}},
\end{equation}
where the last expression uses the antiderivative convention that for 
a function of $t$ denoted by a letter in lower case, the corresponding  
upper case letter\footnote{$\Xi$ is 
upper case $\xi$; $\Lambda$ is upper case $\lambda$; $M$ is upper case 
$\mu$; $N$ is upper case $\nu$; $\Gamma$ is upper case $\gamma$.} 
represents its sum over its domain of definition 
($t=1$ to $T-1$ in this case).  Now
\begin{equation*} 
 \sum_{r\neq j}\tau_{rj} 
 =
 \tau_{j}\sum_{r\neq j}\Xi^k_{rj} 
 =
 \tau_{j}\sum_{t=1}^{T-1}\sum_{r\neq j}\xi^k_{rj}(t).
\end{equation*}
If we denote the probability of {\sl moving\/} away from state $j$ by
\begin{equation*} 
 \mu^k_j(t) = \sum_{r\neq j}\xi^k_{rj}(t),
\end{equation*}
we can write $\sum_{r\neq j}\tau_{rj} = \tau_j M^k_j$.  Substituting 
in eq.~\eqref{sumeq}, this gives:
\begin{equation*} 
 \frac{1}{\tau_{j}} 
 =  
 \sum_{t=1}^{T-1}\frac{f_j(t)\xi^k_{jj}(t)}
                      {1 - f_j(t)M^k_j\tau_{j}},
\end{equation*}
or equivalently:
\begin{equation} \label{sumeqntau}
 1 = \sum_{t=1}^{T-1}\frac{f_j(t)\xi^k_{jj}(t)}
                          {1/\tau_j - f_j(t)M^k_j}.
\end{equation}
We must solve this equation for $\tau_j$, whence we can use 
eq.~\eqref{tautilde} to solve for each of the $\tau_{ij}$.  Since 
$\tau_j$ is nonnegative and constraint~\eqref{maxconstrtau} must 
hold, we recast these conditions in terms of $\tau_j$ as:
\begin{equation*} 
 f^*_j M^k_j < \frac{M_j^k}{\sum_{r\neq j}\tau_{rj}} 
             = \frac{1}{\tau_j} 
             < \infty.
\end{equation*}
Let
\begin{equation*}  
 f^{k *}_{j} 
 = 
 \max_{t\in[T-1]\,\mid\,\xi^k_{jj}(t)\neq0} \: f_{j}(t)\,\,\leq f^*_{j}.
\end{equation*}
Each of the terms in the sum on the right side of 
eq.~\eqref{sumeqntau} is strictly decreasing in $1/\tau_j$ when it is
well-defined ($1/\tau_j > f^{k *}_j M^k_j$).  Values of $1/\tau_j$ 
just larger than $f^{k *}_j M^k_j$ make the sum arbitrarily large, 
and as $1/\tau_j$ increases from that value, the sum decreases 
monotonically to 0, so exactly one value of 
$\tau_j < 1/(f^{k *}_j M^k_j)$ will satisfy eq.~\eqref{sumeqntau}.  We 
can solve the equation numerically to find this value, call it 
$\tau^c_j$.  If $\tau^c_j < 1/(f^*_j M^k_j)$, splitting it 
proportionally to $\Xi_{ij}$ according to eq.~\eqref{tautilde} gives 
the {\sl unique\/} critical point $(\tau_{ij})\in F_j$ of 
$\mathcal{L}(\Theta,{\hat\Theta}^k)$.
 
We can compute explicitly the Hessian of 
$\mathcal{L}(\Theta,{\hat \Theta}^k)$ with respect to the $\tau_{ij}$; 
its components are:
\begin{equation*} 
 \frac{\partial^2\mathcal{L}(\Theta,{\hat\Theta}^k)}
      {\partial\tau_{ij}\partial\tau_{i'j}} 
 = 
 -\frac{\Xi^k_{ij}}{\tau^2_{ij}} \delta_{i i'}
 -\sum_{t=1}^{T-1}\frac{f^2_j(t)\xi^k_{jj}(t)}
                             {(1-f_j(t)\sum_{r\neq j}\tau_{rj})^2}.
\end{equation*}
Thus, as a matrix the Hessian can be written as the sum of two 
matrices:
\begin{equation*}
 \left(\frac{\partial^2\mathcal{L}(\Theta,{\hat\Theta}^k)}
             {\partial\tau_{ij}\partial\tau_{i'j}} 
 \right)
 =
 -\begin{pmatrix}
   \Xi^k_{1j}/\tau^2_{1j} \\
   &\ddots\\
   &&\Xi^k_{Nj}/\tau^2_{Nj} 
  \end{pmatrix}
 -\sum_{t=1}^{T-1}\frac{f^2_j(t)\xi^k_{jj}(t)}
                       {(1-f_j(t)\sum_{r\neq j}\tau_{rj})^2}\,
 {\bf 1}{\bf 1}^\mathsf{T},
\end{equation*}
where ${\bf 1}\in\mathbb{R}^N$ is the vector of all $1$s.  Each of the 
matrices on the right is negative semi-definite, so the Hessian is 
also.  Thus the (unique) critical point we found in this case is a 
global maximum of $\mathcal{L}(\Theta,{\hat \Theta}^k)$ in $F_j$ and 
hence the choice for $\hat\tau_{ij}^{k+1}$.

If the solution does not satisfy the original 
constraint~\eqref{maxconstrtau}, {\it i.e.}, 
$\tau^c_j \ge  1/(f^*_j M^k_j)$, or if the right side of 
eq.~\eqref{sumeq} is $0$, the maximum will be on the boundary of 
$F_j$.  Thus we maximize $\mathcal{L}(\Theta,{\hat \Theta}^k)$ subject 
to the boundary constraint
\begin{equation*}
 \sum_{j\neq i\in[N]} \tau_{ij} = \frac{1}{f^*_j}.                                                                
\end{equation*}

This is in the form of the constraint in the regular Baum-Welch 
algorithm with 1 replaced by $1/f^*_j$ and the self-transition 
probability set to 0.  Thus the critical $\tau_{ij}$ can be computed 
as in the Baum-Welch algorithm,\footnote{In our notation, the usual 
Baum-Welch estimate is $\hat a^k_{ij} = \hat \tau^k_{ij} 
= \Xi^k_{ij}/\Gamma^k_j$ for all $i, j$.} with $\Gamma^k_j$ replaced 
by $M^k_j$ and the solution divided by $f^*_j$:
\begin{equation*} 
 \tau^c_{ij} = \frac{\Xi^k_{ij}}{f^*_j  M^k_j}.
\end{equation*}
This is the unique critical point in this case, and the global maximum 
of $\mathcal{L}(\Theta,{\hat \Theta}^k)$ in $F_j$, by the same 
argument as in the Baum-Welch algorithm. This becomes the choice for 
$\hat\tau_{ij}^{k+1}$.

\smallskip
We  turn to the computation of $\hat\epsilon_{sj}^{k+1}$, 
$s\in[M]$.  As before, we begin by finding the stationary points of 
$\mathcal{L}(\Theta,{\hat\Theta}^k)$, now relative to 
$(\epsilon_{sj})\in G_j\subset\mathbb{R}^{M}$ defined 
by constraints~\eqref{maxconstreps} and the non-negativity of these parameters.  
Using eqs.~\eqref{Lexp} and~\eqref{bsj} gives:
\begin{equation} \label{partialeps}
 \frac{\partial\mathcal{L}(\Theta,{\hat \Theta}^k)}
      {\partial \epsilon_{s j}} 
 = 
 \frac{1}{\epsilon_{s j}} 
 \sum_{t=1}^{T} \gamma^k_j(t)\delta_{s,y_t}
       -\sum_{t=1}^{T} \frac{g_{j}(t)\gamma^k_j(t)}
             {1-g_{j}(t)\sum_{l\in[M]}\epsilon_{lj}}\delta_{0,y_t}.
\end{equation}
As we did for eq.~\eqref{partialtau}, we must consider the situations
when either of the sums in eq.~\eqref{partialeps} vanishes.  When the 
left sum is $0$, a extreme value is given by $\epsilon_{sj}=0$, and 
when the right sum is $0$, constraint~\eqref{maxconstreps} is 
saturated.  Assuming neither of the sums vanishes, we  find the 
stationary points by solving
\begin{equation} \label{zeroeqneps}
 \frac{1}
      {\epsilon_{s j}} \sum_{t=1}^{T} \gamma^k_j(t)\delta_{s,y_t}
 = 
 \sum_{t=1}^{T} \frac{g_{j}(t)\gamma^k_j(t)}
                     {1-g_{j}(t)\sum_{l\in[M]}\epsilon_{lj}}\delta_{0,y_t}.
\end{equation}
Solution to the {\sl emission\/} equations, eqs.~\eqref{zeroeqneps},
follows using the same steps as for the {\sl transition\/} 
equations, eqs.~\eqref{sumeq}.  We first denote the probability of 
emission $s\in[M]$ from state $j$ by:\footnote{In a simplified model 
for the mobile phone data we discussed in the Introduction, $M=N$ and 
every state $j$ emits either the signal $j$ or $0$; in other words, 
$\epsilon_{sj} = 0$ if $s\ne j\in[M]$.  This simplifies the following
expressions:  For $t$ such that $y_t = j$, $\gamma^k_j(t) = 1$ (the 
state is $j$ with certainty if the observed emission is $j$), 
{\it i.e.},  $\lambda^k_{jj}(t) = \delta_{j,y_t}$, and 
$\lambda^k_{sj}(t) = 0$ if $s\ne j\in[M]$.}
\begin{equation*} 
 \lambda^k_{sj}(t) = \gamma^k_j(t) \delta_{s,y_t},
\end{equation*} 
in terms of which we rewrite eq.~\eqref{zeroeqneps} as
\begin{equation} \label{firsteqneps}
\frac{1}{\epsilon_{s j}} \sum_{t=1}^{T} 
  \lambda^k_{sj}(t)
 =  
 \sum_{t=1}^{T}\frac{g_j(t)\lambda^k_{0j}(t)}
                      {1-g_{j}(t)\sum_{l\in[M]}\epsilon_{lj}}.
\end{equation}
We define (independent of $s$)
\begin{equation} \label{epstilde}
 \epsilon_{j} = \frac{\epsilon_{sj}}{\Lambda^k_{sj}}.
\end{equation}
We also define the probability of any {\sl non-zero\/} emission from 
state $j$,
\begin{equation*} 
 \nu^k_j(t) = \sum_{l\in [M]}\lambda^k_{lj}(t),
\end{equation*} 
which, used in eq.~\eqref{firsteqneps}, gives
\begin{equation} \label{sumeqneps}
 1 = 
 \sum_{t=1}^{T}\frac{g_j(t)\lambda^k_{0j}(t)}
                      {1/\epsilon_j - g_j(t) N^k_j}.
\end{equation}
Let  
\begin{equation*}  
 g^{k *}_{j} 
 = 
 \max_{t\in [T]\,\mid\,\lambda^k_{0j}(t)\neq0}\:g_{j}(t)\,\,\leq g^*_{j}.
\end{equation*}
As before, there is exactly one solution 
$\epsilon^c_j < 1/(g^{k *}_j N^k_j)$ to eq.~\eqref{sumeqneps}.  We can 
find it numerically, and if $\epsilon^c_j < 1/(g^*_j N^k_j)$, we use 
eq.~\eqref{epstilde} to find all the $\epsilon_{sj}$, which will then
satisfy constraint~\eqref{maxconstreps}.  Again we can compute the Hessian 
explicitly to confirm that this is a global maximum of 
$\mathcal{L}(\Theta,{\hat \Theta}^k)$, now in $G_j$, and hence the 
choice for $\hat\epsilon_{s j}^{k+1}$.

If $\epsilon^c_j \ge 1/(g^*_j N^k_j)$, or if the right side of 
eq.~\eqref{firsteqneps} is $0$, we must find instead the critical 
point on the boundary of $G_j$:
\begin{equation*}
 \sum_{s\in [M]} \epsilon_{sj} = \frac{1}{g^*_j}.                                                                
\end{equation*}
Again, this is in the form of the constraint in the regular Baum-Welch 
algorithm.  Accordingly, we set the critical $ \epsilon_{sj}$ to the 
Baum-Welch estimate, rescaled by $g^*_j$:
\begin{equation*} 
 \epsilon^c_{sj} = \frac{\Lambda^k_{s j}}{g^*_j N^k_j}.
\end{equation*}
This is the unique critical point and the global maximum of 
$\mathcal{L}(\Theta,{\hat \Theta}^k)$ in $G_j$, and therefore the 
choice for $\hat\epsilon_{sj}^{k+1}$ in this case.

\smallskip
We have shown how to find ${\hat\Theta}^{k+1}$ maximizing 
$\mathcal{L}(\Theta,{\hat\Theta}^k)$ in eq.~\eqref{Lexp}.  This 
algorithm converges as claimed because it is an instance of 
expectation maximization.  To understand its time complexity we must 
consider the numerical solution of eqs.~\eqref{sumeqntau} 
and~\eqref{sumeqneps}.  To simplify notation we rewrite 
eq.~\eqref{sumeqntau} in terms of $u = 1/\tau_j$, and a function 
$w(u)$,
\begin{equation} \label{sumeqneta}
w(u) = \sum_{t=1}^{T-1}\frac{f_j(t)\xi^k_{jj}(t)}
                            {u  - f_j(t)M^k_j} 
       -1,
\end{equation}
as $w(u) = 0$.  As an initial estimate for the root, $u^c$, we can 
use $u^L > f^{k *}_j M^k_j$ such that 
\begin{equation} \label{initest}
 \frac{f^{k *}_j \xi^k_{jj}(t^{k *}_j)}{u^L - f^{k *}_jM^k_j} 
 = 1,
\end{equation}
where $t^{k *}_j \in [T-1]$ satisfies $f_j(t^{k *}_j) = f^{k *}_j$ and 
$\xi^k_{jj}(t^{k *}_j)\neq0$.  $u^L \le u^c$ since we found it using 
only one of the nonnegative terms in the sum in eq.~\eqref{sumeqneta}.

Now recall that $f_j(t)\xi^k_{jj}(t) \le f^{k *}_j$ for $t \in [T-1]$, 
and 
\begin{equation*}                                                      
 0 \le M^k_j 
 = 
 \sum_{t\in [T-1]} \mu^k_j(t) 
 = 
 \sum_{t\in [T-1]} \sum_{r\neq j}\xi^k_{rj}(t)
 \le 
 \sum_{t\in [T-1]} \gamma^k_{j}(t) 
 \le 
 T-1.                                                
\end{equation*}
Thus each term in the sum in eq.~\eqref{sumeqneta} is no more than 
\begin{equation*} 
\frac{f^{k *}_j}{u - (T-1)f^{k *}_j}.
\end{equation*}
At $u^R = 2(T-1)f^{k *}_j$ this is $1/(T-1)$, so $w(u^R) \le 0$, 
which implies
\begin{equation}\label{etajcbound}
 u^L \le u^c \le u^R = 2(T-1)f^{k *}_j.
\end{equation}

Using Newton's method, once we have an initial estimate ``sufficiently 
close'' to the root of $w(u) = 0$, the time complexity to find it with 
error less than $\epsilon$ is $O\bigl(T\log(1/\epsilon)\bigr)$,
where the $T$ comes from the cost of evaluating $w(u)$ and $w'(u)$ at
each iteration; in practice this is how we would find the root.  Since 
the length of the interval in \eqref{etajcbound} is $O(T)$, however,
the bisection method gets us to precision $\epsilon$ with 
$O\bigl(\log(T/\epsilon)\bigr)$ steps, with total cost 
$O\bigl(T\log (T/\epsilon)\bigr)$; thus this is the total complexity.

We need to solve  eqs.~\eqref{sumeqntau} and~\eqref{sumeqneps} 
$N$ and $M$ times, respectively, at each iteration, which thus adds 
$O\bigl((N + M)T \log (T/\epsilon)\bigr)$ to the $O(N^2T)$ 
complexity of the computations for $\gamma^{k+1}_i$ and 
$\xi^{k+1}_{ij}$.  Thus the time complexity for the whole algorithm is 
$O\bigl((N^2+M) T \log (T/\epsilon)\bigr) $.  
\end{proof}

\section{Numerical Simulations}

To demonstrate the effect of the activity functions we consider a 
simple model with $N = 3$ states and the same number of possible 
emissions ($M=3$).  From any state $j$, we only allow an emission to 
be either its own label $j$ or $0$, {\it i.e.}, $\epsilon_{s j}=0$ for 
$j\neq s\in[M]$, so a non-zero emission uniquely identifies the state 
that emits it.  We choose random transition and emission parameters:
\begin{align*}
 (\epsilon_{jj}) &= (0.770347,0.579213,0.0821789);                  \\
 (\tau_{ij})
 &= \left(\begin{array}{ccc}
          & 0.298244 & 0.0621274 \\
          0.134788 & & 0.3710750 \\
          0.383490 & 0.182008 &  \\
          \end{array}
    \right),
\end{align*}
where the omitted values are the components for which $i = j$. 

We generate sequences of length $T=24\cdot6\cdot7\cdot200$ (we may 
think of this as $200$ weeks, with an observation every $10$ minutes).  
We consider activity functions with variations that may approximate 
observed data, {\it i.e.}, periodic variations with a period of 
$24\cdot6$ (one day).  Specifically, our numerical simulations use the
following three functions:
\begin{enumerate}[label=(\roman{*})]
\item \label{} constant function,
\begin{equation*} \label{likerat1}
 \mathit{1}(t) = 1;
\end{equation*} 
\item \label{} raised cosine,
\begin{equation*} \label{likerat1}
 r_n(t) = \frac{n-\cos\left(2\pi t/(24\cdot6)
                     \right)}
               {n+1};
\end{equation*} 
\item \label{} shifted cosine
\begin{equation*} \label{likerat2}
 c_j(t) 
 = 
 \frac{1}{3} \left[2-\cos\left(\frac{2\pi(t-6j)}
                                    {24\cdot6} 
                        \right)
            \right].
\end{equation*}
\end{enumerate}

We generate a random sequence of states, $x$, and resulting emissions, 
$y$, using the transition and emission parameters above, and a pair of 
activity functions (a list of these pairs is shown in 
Table~\ref{table2}).  

Before computing the sequence of parameter estimates, we need to 
specify how we compute initial estimates to start the iteration.  This 
can only depend on the observed emission sequence $y$, since in any 
real scenario $x$ is unknown.  As a first guess, for this simple 
model, we interpolate the state sequence $x$ as follows:\footnote{For 
more general models, finding initial parameter estimates will be more
complicated, depending on the particulars of the model.} For every 
pair of successive non-zero emissions, there is a segment of zeros (no 
emission) separating them.  We divide each such segment into two 
subsegments:  Let $j\in[M]$ be the emission immediately preceding the 
segment, and $i\in[M]$ be the emission immediately following the 
segment.  The second subsegment starts at the first time step after 
the one where $f_j$ first attains its maximum value on the segment 
({\it i.e.}, a time at which there is the maximum probability of 
hopping from state $j$ to state $i$).  We assign state $j$ to the time 
steps in the first subsegment and the state $i$ to those in the 
second.  If the emission sequence $y$ starts with a segment of zeros, 
then that segment is assigned the value of the first non-zero 
emission; similarly a terminal sequence of zeros is given the value of 
the last non-zero emission.  Denote the interpolated states by 
$z = (z_t)$, $t\in[T]$.  

From $z$, we compute the estimate $(\hat\pi^1_j)$ for the initial 
distribution over the states $(\pi_j)$ by their frequencies of 
occurrence.  For the initial $\tau_{ij}$ estimate, $\hat\tau^1_{ij}$, 
we use the method described in the proof of Theorem~\ref{mainthm} to 
solve eq.~\eqref{partialtau}, using 
$\xi^k_{ij}(t) = \delta_{i,z_{t+1}}\delta_{j,z_t}$, including $i = j$.  
For the initial $\epsilon_{jj}$ estimate, $\hat\epsilon^1_{jj}$, we 
also use the method described in the proof of Theorem~\ref{mainthm} to 
solve eq.~\eqref{partialeps}, using $\gamma_j(t) = \delta_{j,z_t}$.

To understand the performance of the algorithm in 
Theorem~\ref{mainthm}, we need a measure of the error between the 
estimates and the real parameter values.  The {\sl relative entropy\/}
is one measure for a stationary HMM.  In our case we need to account 
for the time variation of the transition and emission probabilities, 
$a_{ij}(t)$ and $b_{sj}(t)$.  Hence we define a modified version of a 
relative entropy error criterion, the {\sl averaged relative entropy}.
\begin{defn}
Let $\mathcal{Q}=(Q_t)$ and $\mathcal{P}=(P_t)$, $t\in[T]$, be two 
finite sequences of discrete probability distributions on a finite 
set $\mathcal{I}$.  The {\sl Averaged Relative Entropy\/} 
($\mathsf{ARE}$) 
of $\mathcal{P}$ with respect to $\mathcal{Q}$ is
\begin{equation*}
 \mathsf{ARE}(\mathcal{P},\mathcal{Q}) 
 = 
 \frac{1}{T}\sum_{t\in[T]}\mathsf{RE}(P_t,Q_t),
\end{equation*}
where the usual relative entropy ($\mathsf{RE}$) is given by
\begin{equation*}
 \mathsf{RE}(P_t,Q_t) 
 = 
 \sum_{i\in\mathcal{I}} P_t(i) \log \frac{P_t(i)}{Q_t(i)}.
\end{equation*}
\end{defn}

Thus the error function that we compute for given $(\tau_{ij})$
and estimate $(\hat \tau^k_{ij})$ is
\begin{equation*}
 \mathcal{E}_{\tau}\big((\tau_{ij}),(\hat\tau^k_{ij})\big)
 =
 \mathsf{ARE}\Big(\big(a_{ij}(t)\big),\big(\hat a^k_{ij}(t)\big)\Big),
\end{equation*}
where $a_{ij}(t)$ and $\hat a^k_{ij}(t)$ are related through $f_j$ to $\tau_{ij}$ 
and $\hat\tau^k_{ij}$  by eq.~\eqref{aij} and 
eq.~\eqref{hataij}, respectively.  Similarly, for $(\epsilon_{s j})$ and estimate $(\hat\epsilon^k_{sj})$, the error function is 
\begin{equation*}
 \mathcal{E}_{\epsilon}\big((\epsilon_{s j}),(\hat\epsilon^k_{sj})\big)
 =
 \mathsf{ARE}\Big(\big(b_{sj}(t)\big),\big(\hat b^k_{sj}(t)\big)\Big),
\end{equation*}
where $b_{sj}(t)$ and $\hat b^k_{sj}(t)$ are related through $g_j$ to 
$\epsilon_{s j}$ and $\hat\epsilon^k_{sj}$  by eq.~\eqref{bsj} 
and eq.~\eqref{hatbsj}, respectively.  (Remember that we are 
considering the simple case in which the emission $y_t = s\in\{0,j\}$ when the state $x_t=j$.)  

The pairs of activity functions $f_j$ and $g_j$ that we simulate 
numerically are described in Table~\ref{table1}, where the column 
indices label these pairs.

\begin{table}   [H]
\begin{center}
\begin{tabu}to\linewidth{|[1.5pt]c|[1.5pt]c|c|c|c|c|c|c|c|[1.5pt]}  %
\tabucline[1.5pt]-
& a & b & c & d & e & f & g & h \\
\tabucline[1.5pt]-
$f_j$& $\mathit{1}$&$\mathit{1}$& $\mathit{1}$& $r_1$& $c_j$ &$r_2$& $r_1$& $c_j$  \\
\hline 
$g_j$&$c_j$&$r_1$&$\mathit{1}$&$r_1$&$c_j$&$\mathit{1}$&$\mathit{1}$&$\mathit{1}$\\ 
\tabucline[1.5pt]-
\end{tabu}
\end{center}
\caption{Functions used in numerical simulations}
 \label{table1}  
\end{table} 

For each set of pairs of activity functions in Table~\ref{table1} we run the 
algorithm for $50$ iterations.  Figures~\ref{figtrans} 
and~\ref{figemis} plot the averaged relative entropy for the parameter
estimates as a function of iteration step.  The labels (a)--(h) to the 
right of each plot appear in the order of the final error values.  We 
do not provide a plot showing convergence of $(\hat\pi^k_j)$ since the 
only noticeable trend is that if they converge to an exact state 
value, it is usually to the initial state of the interpolated sequence 
$z$. 

In each case the error for both the transitions and the emissions 
decreases to small values.  Since the $\mathsf{ARE}$ depends on the 
activity functions as well as on the parameters and their estimates, 
we need to compute a baseline error value for each case.  For the
parameters $(\tau_{ij})$ and a specific choice of $(f_j)$ it is:
\begin{equation*} 
 \mathcal{B}_\tau\big((\tau_{ij}),(f_j)\big)
 =
 \mathsf{E}\Big[\mathsf{ARE}\Big(\big(a_{ij}(t)\big),
                                 \big( a'_{ij}(t)\big)
                            \Big)
           \Big],
\end{equation*}
where $a_{ij}(t)$ and $a'_{ij}(t)$ are related through $f_j$ to $\tau_{ij}$ and 
$\tau'_{ij}$, respectively,  by eq.~\eqref{aij}, and where 
$\mathsf{E}[\cdot]$ denotes the expectation over uniformly random
$(\tau'_{ij})$.  To estimate this expectation value, we compute the 
average $\mathsf{ARE}$ of the parameters with respect to 1000 
independently chosen sets of random parameters (rather than their 
estimates from our algorithm), for each case (a)--(h).  For the 
emission parameters we compute baselines the same way, using 1000 
uniformly random values $(\epsilon'_{sj})$ to estimate
\begin{equation*} 
 \mathcal{B}_\epsilon\big((\epsilon_{s j}),(g_j)\big)
 =
 \mathsf{E}\Big[\mathsf{ARE}\Big(\big(b_{sj}(t)\big),
                                 \big(b'_{sj}(t)\big)
                            \Big)
           \Big],
\end{equation*}
where $b_{sj}(t)$ and $b'_{sj}(t)$ are related through $g_j$ to $\epsilon_{s j}$ 
and $\epsilon'_{sj}$, respectively,  by eq.~\eqref{bsj}, and where 
$\mathsf{E}[\cdot]$ denotes the expectation value over uniformly
random $(\epsilon'_{sj})$.  The baseline averages thus obtained for 
function pairs in Table~\ref{table1} are recorded in 
Table~\ref{table2}, where the row labels indicate the parameters 
being baselined.
\begin{table}  [H]
\begin{center}
\begin{tabu}to\linewidth{|[1.5pt]c|[1.5pt]c|c|c|c|c|c|c|c|[1.5pt]}  
\tabucline[1.5pt]-
& a& b& c & d  & e& f  & g & h \\
\tabucline[1.5pt]-
$\mathcal{B}_\tau\big((\tau_{ij}),(f_j)\big)$&$1.637$& $1.677$&$1.657$& $0.615$&$0.603$& $0.811$&$0.81$ & $0.610$  \\
\hline
$\mathcal{B}_\epsilon\big((\epsilon_{s j}),(g_j)\big)$&$0.603$&$0.578$&$1.563$&$0.592$&$0.584$&$1.466$&$1.539$&$1.534$\\ 
\tabucline[1.5pt]-
\end{tabu}
\end{center}
\caption{Baseline errors for $(\tau_{ij})$ and $(\epsilon_{s j})$ for 
activity function pairs from Table~\ref{table1}.} 
 \label{table2}
\end{table}

We plot these baseline errors as horizontal lines in 
Figures~\ref{figtrans} and \ref{figemis}.  Most of these are too close 
to be distinguishable; indeed they are all $O(1)$, in contrast to the 
estimation errors plots which are almost all smaller by at least an
order of magnitude, and in most cases by 3 or 4, indicating very good
parameter estimates.  Furthermore, the relative quality of the 
estimates can be understood:  Case (c) is the standard HMM, for which
our algorithm reduces to the Baum-Welch 
algorithm~\cite{BaumEagon,Welch}.  Cases (a) and (b) have greater 
errors, which is not surprising since they have non-constant emission 
activity functions, oscillating in value up to 1.  This means that for
each of these cases, non-zero emissions are lower probability events, 
so there is less information in $y$.  Possibly surprising is the fact 
that when the transition activity function is non-constant, cases 
(d)--(h), the errors are {\sl smaller\/} than in the standard HMM 
case.  But this happens because state changing transitions are 
reduced, so that each non-zero emission observed provides more 
information.  And among these cases, those with varying emission 
activity levels have larger errors than those without.

\begin{figure}[H] 
\includegraphics[scale=0.9]{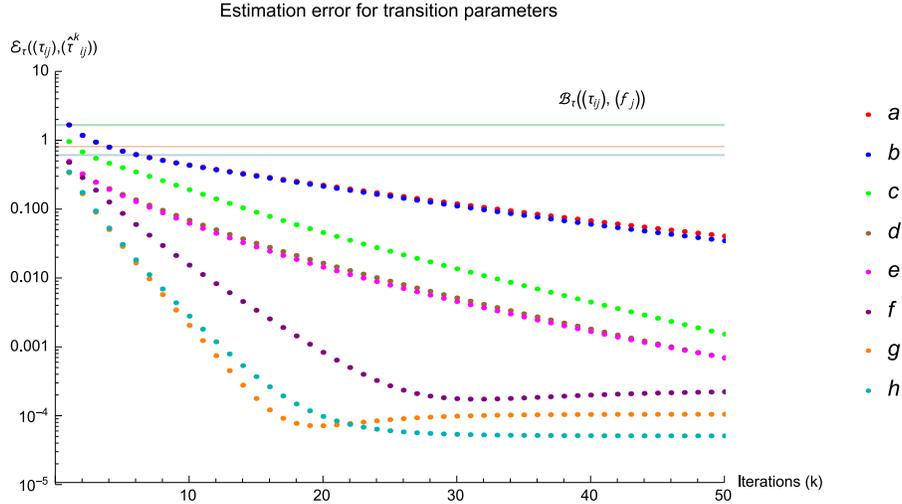}
\caption{Transition 
parameters estimation error,  $\mathcal{E}_{\tau}\big((\tau_{ij}),(\hat\tau^k_{ij})\big)$, for successive iterations $k$.
Labels to the right are of activity function pairs from 
Table~\ref{table1}, and are displayed in the order of the final error 
values. Baseline errors, $\mathcal{B}_\tau\big((\tau_{ij}),(f_j)\big)$, from Table~\ref{table2} are shown as horizontal lines.}
\label{figtrans}
\end{figure}

\begin{figure}[H] 
\includegraphics[scale=0.9]{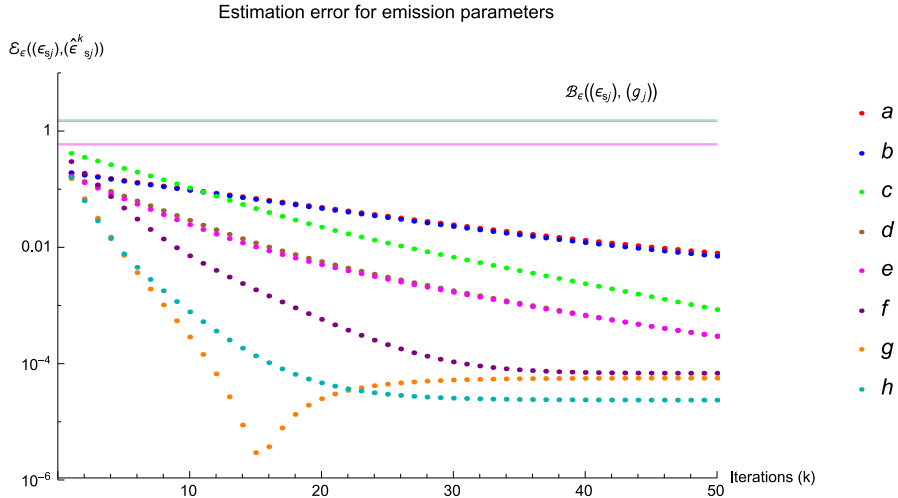}
\caption{Emission parameters estimation error, $\mathcal{E}_{\epsilon}\big((\epsilon_{s j}),(\hat\epsilon^k_{sj})\big)$, for successive iterations $k$.  
Labels to the right are of activity function pairs from 
Table~\ref{table1}, and are displayed in the order of the final error 
values.  Baseline errors,  $\mathcal{B}_\epsilon\big((\epsilon_{s j}),(g_j)\big)$, from Table~\ref{table2} are shown as horizontal lines.}
\label{figemis}
\end{figure}

\setlength{\arrayrulewidth}{1mm}
\setlength{\tabcolsep}{18pt}
\renewcommand{\arraystretch}{1.5}
 
\def\MR#1{\relax\ifhmode\unskip\spacefactor3000 \space\fi
  \href{http://www.ams.org/mathscinet-getitem?mr=#1}{MR#1}}

\end{document}